\documentclass[journal]{IEEEtran}
\ifCLASSINFOpdf
\else
\fi

\usepackage{epsfig, graphicx}
\usepackage{tikz}
\usepackage[colorlinks,
linkcolor=black,
anchorcolor=black,
citecolor=black
]{hyperref}
\usepackage{amsmath, amsthm, amssymb}
\usepackage{cite}
\usepackage{url}
\usepackage{diagbox}

\newtheorem{prop}{Proposition}
\newtheorem{theorem}{Theorem}
\newtheorem{lemma}{Lemma}
\newtheorem{definition}{Definition}

\usepackage{algorithm}
\usepackage{algorithmic}

\begin{document}
%
\title{\LARGE Analyzing Upper Bounds on Mean Absolute Errors for Deep Neural Network Based Vector-to-Vector Regression}
%
%
%

\author{Jun Qi,~\IEEEmembership{Student Member,~IEEE},
	 Jun Du,~\IEEEmembership{Member,~IEEE},
	 Sabato Marco Siniscalchi,~\IEEEmembership{Senior Member,~IEEE},\\
	 Xiaoli Ma,~\IEEEmembership{Fellow,~IEEE},
         and~Chin-Hui Lee,~\IEEEmembership{Fellow,~IEEE}
\thanks{J. Qi, X. Ma and C.-H. Lee are with the School of Electrical and Computer Engineering, Georgia Institute of Technology, Atlanta,
GA, 30332 USA e-mail: (qij41@gatech.edu, xiaoli@gatech.edu, chl@ece.gatech.edu).}
\thanks{J. Du is with the National Engineering Laboratory for Speech and Language Information Processing, University of Science and Technology of China, Hefei 230027, China (e-mail: jundu@ustc.edu.cn).}
\thanks{S. M. Siniscalchi is with the Faculty of Architecture and Engineering, University of Enna ``Kore", Enna 94100, Italy, and also with the Georgia Institute of Technology, Atlanta, GA 30332 USA (e-mail: marco.siniscalchi@unikore.it).}
}
\maketitle

\begin{abstract}
In this paper, we show that, in vector-to-vector regression utilizing deep neural networks (DNNs), a generalized loss of mean absolute error (MAE) between the predicted and expected feature vectors is upper bounded by the sum of an approximation error, an estimation error, and an optimization error. Leveraging upon error decomposition techniques in statistical learning theory and non-convex optimization theory, we derive upper bounds for each of the three aforementioned errors and impose necessary constraints on DNN models. Moreover, we assess our theoretical results through a set of image de-noising and speech enhancement experiments. Our proposed upper bounds of MAE for DNN based vector-to-vector regression are corroborated by the experimental results and the upper bounds are valid with and without the ``over-parametrization'' technique. 
\end{abstract}


%
\IEEEpeerreviewmaketitle

\section{Introduction}
\label{sec1}
%
%
%
%

\IEEEPARstart{V}ector-to-vector regression, also known as multivariate regression, provides an effective way to find underlying relationships between input vectors and their corresponding output ones at the same time. The problems of vector-to-vector regression are of great interest in signal processing, wireless communication, and machine learning communities. For example, speech enhancement aims at finding a vector-to-vector mapping to convert noisy speech spectral vectors to the clean ones~\cite{xu2015regression, jun2020}. Similarly, clean images can be extracted from the corrupted ones by leveraging upon image de-noising techniques~\cite{xie2012image}. Besides, wireless communication systems are designed to transmit local encrypted and corrupted codes to targeted receivers with decrypted information as correct as possible~\cite{sanchez2004svm, el2011network}. Moreover, the vector-to-vector regression tasks are also commonly seen in ecological modeling, natural gas demand forecasting, and drug efficacy prediction domains~\cite{borchani2015}.  

The vector-to-vector regression can be theoretically formulated as follows: given a $d$-dimensional input vector space $\mathbb{R}^{d}$ and a measurable $q$-dimensional output vector space $\mathbb{R}^{q}$, the goal of vector-to-vector regression is to learn a functional relationship $f: \mathbb{R}^{d} \rightarrow \mathbb{R}^{q}$ such that the output vectors can approximate desirable target ones. The regression process is described as:
\begin{equation}
\label{eq:basic}
\textbf{y} = f(\textbf{x}) + \textbf{e},
\end{equation}
where $\textbf{x} \in \mathbb{R}^{d}$, $\textbf{y} \in \mathbb{R}^{q}$, $\textbf{e}$ is an error vector, and $f$ refers to the regression function to be exploited. To implement the regression function $f$, linear regression~\cite{pearson1896} was the earliest approach and several other methods, such as support vector regression~\cite{vapnik1997} and decision tree regressions~\cite{milstein1975}, were further proposed to enhance regression performance. However, deep neural networks (DNN)~\cite{hinton:basic06, lecun2015deep} with multiple hidden layers offer a more efficient and robust solution to dealing with large-scale regression problems. For example, our previous experimental study~\cite{xu2013experimental} demonstrated that DNNs outperform shallow neural networks on speech enhancement. Similarly, auto-encoders with deep learning architectures can achieve better results on image de-noising~\cite{zhang2017beyond}. 

Although most endeavors on DNN based vector-to-vector regression focus on the experimental gain in terms of mapping accuracy, the related theoretical performance of DNN has not been fully developed. Our recent work~\cite{qi2019theory} tried to bridge the gap by analyzing the representation power of DNN based vector-to-vector regression and deriving upper bounds for different DNN architectures. However, those bounds particularly target experiments with consistent training and testing conditions, and they may not be adapted to the experimental tasks where unseen testing data are involved. Therefore, in this work, we focus on an analysis of the generalization power and investigate upper bounds on a generalized loss of mean absolute error (MAE) for DNN based vector-to-vector regression with mismatched training and testing scenarios. Moreover, we associate the required constraints with DNN models to attain the upper bounds. 

The remainder of this paper is organized as follows: Section~\ref{sec1} highlights the contribution of our work and its relationship with the related work. Section~\ref{sec2} underpins concepts and notations used in this work. Section~\ref{sec3} discusses the upper bounds on MAE for DNN based vector-to-vector regression by analyzing the approximation, estimation, and optimization errors, respectively. Section~\ref{sec:mse} presents how to utilize our derived upper bounds to estimate practical MAE values. Section~\ref{sec:exp} shows the experiments of image de-noising and speech enhancement to validate our theorems. Finally, Section \ref{sec:con} concludes our work.

\section{Related Work and Contribution}
\label{sec1}

The recent success of deep learning has inspired many studies on the expressive power of DNNs~\cite{related1, related2, related3, rolnick2018}, which extended the original universal approximation theory on shallow artificial neural networks (ANNs) \cite{kolmogorov, cybenko, barron, barron94, hornik} to DNNs. As discussed in~\cite{fan2019selective}, the approximation error is tightly associated with the DNN expressive power. Moreover, the estimation error and optimization error jointly represent the DNN generalization power, which can be reflected by error bounds on the out-of-sample error or the testing error. The methods of analyzing DNN generalization power are mainly divided into two classes: one refers to algorithm-independent controls~\cite{neyshabur2015norm, bartlett2017spectrally, golowich2018} and another one denotes algorithm-dependent approaches~\cite{mei2018mean,chizat2018global}. In the class of algorithm-independent controls, the upper bounds for the estimation error are based on the empirical Rademacher complexity~\cite{bartlett2002rademacher} for a functional family of certain DNNs. In practice, those approaches concentrate on techniques of how weight regularization affects the generalization error without considering advanced optimizers and the configuration of hyper-parameters. As for the algorithm-dependent approaches~\cite{mei2018mean,chizat2018global}, several theoretical studies focus on the ``over-parametrization'' technique~\cite{du2018power, neyshabur2018towards, allen2018convergence, li2018learning}, and they suggest that a global optimal point can be ensured if parameters of a neural network significantly exceed the amount of training data during the training process.

We notice that the generalization capability of deep models can also be investigated through the stability of the optimization algorithms. More specifically, an algorithm is stable if a small perturbation to the input does not significantly alter the output, and a precise connection between stability and generalization power can be found in \cite{bousquet2002, huck}. Besides, in \cite{arbi2018b}, the authors investigate the stability and oscillations of various competitive neural networks from the perspective of equilibrium points. However, the analysis of the stability of the optimization algorithm is out of the scope of the present work, and we do not discuss it further in this study. 

In this paper, the aforementioned issues are taken into account by employing the error decomposition technique~\cite{devroye2013probabilistic} with respect to an empirical risk minimizer (ERM)~\cite{vapnik2013nature, vapnik1992principles} using three error terms: an approximation error, an estimation error, and an optimization error. Then, we analyze generalized error bounds on MAE for DNN based vector-to-vector regression models. More specifically, the approximation error can be upper bounded by modifying our previous bound on the representation power of DNN based vector-to-vector regression~\cite{qi2019theory}. The upper bound on the estimation error relies on the empirical Rademacher complexity~\cite{bartlett2002rademacher} and necessary constraints imposed upon DNN parameters. The optimization error can be upper bounded by assuming $\gamma$-Polyak-Lojasiewicz ($\gamma$-PL)~\cite{karimi2016linear} condition under the ``over-parameterization'' configuration for neural networks~\cite{allen2018learning, vaswani2018fast}. Putting together all pieces, we attain an aggregated upper bound on MAE by summing the three upper bounds. Furthermore, we exploit our derived upper bounds to estimate practical MAE values in experiments of DNN based vector-to-vector regression. 

We use image de-noising and speech enhancement experiments to validate the theoretical results in this work. Image de-noising is a simple regression task from $[0, 1]^{d}$ to $[0, 1]^{d}$, where the configuration of ``over-parametrization'' can be simply satisfied on datasets like MNIST~\cite{deng2012mnist}. Speech enhancement is another useful illustration of the general theoretical analysis because it is an unbounded conversion from $\mathbb{R}^{d} \rightarrow \mathbb{R}^{q}$. Although the ``over-parametrization'' technique could not be employed in the speech enhancement task due to a significantly huge amount of training data, we can relax the ``over-parametrization'' setup and solely assume the $\gamma$-PL condition to attain the upper bound for MAE. In doing so, the upper bound can be adopted in experiments of speech enhancement.

\section{Preliminaries}
\label{sec2}

\subsection{Notations}
\begin{itemize}
\item $f\circ g$:  The composition of functions $f$ and $g$.
\item $||\textbf{v}||_{p}$: $L_{p}$ norm of the vector $\textbf{v}$.
\item $\langle \textbf{x}, \textbf{y} \rangle$ and $\textbf{x}^{T}\textbf{y}$: Inner product of two vectors $\textbf{x}$ and $\textbf{y}$.
\item $[q]$: An integer set $\{1, 2, 3, ..., q\}$.
\item $\nabla f$: A first-order gradient of function $f$.
\item $\mathbb{E}[X]$: Expectation over a random variable $X$.
\item $w_{j}$: The $j$-th element in the vector $\textbf{w}$.
\item $f_{v}$: DNN based vector-to-vector regression function. 
\item $g_{u}$: Smooth ReLU function.
\item $\textbf{1}$: A vector of all ones. 
\item $\textbf{1}_{m}$: Indicator vector of zeros but with the $m$-th dimension assigned to $1$. 
\item $\mathbb{R}^{d}$: $d$-dimensional real coordinate space. 
\item $\mathbb{F}$: A family of the DNN based vector-to-vector functions. 
\item $\mathbb{L}$: A family of generalized MAE loss functions.
\end{itemize}

\subsection{Numerical Linear Algebra}
\begin{itemize}
\item H\"{o}lder's inequality: Let $p, q \ge 1$ be conjugate: $\frac{1}{p} + \frac{1}{q} = 1$. Then, for all $\textbf{x}, \textbf{y} \in \mathbb{R}^{n}$,
\begin{equation}
|\langle \textbf{x}, \textbf{y} \rangle| \le ||\textbf{x}||_{p} ||\textbf{y}||_{q},
\end{equation}
with equality when $|y_{i}| = |x_{i}|^{p-1}$ for all $i\in [N]$. In particular, when $p = q = 2$, H\"{o}lder's inequality becomes the Cauchy-Shwartz inequality.
\end{itemize}

\subsection{Convex and Non-Convex Optimization}
\begin{itemize}
\item A function $f$ is $\beta$-Lipschitz continuous if $\forall \textbf{x}, \textbf{y} \in \mathbb{R}^{n}$, 
\begin{equation}
\label{eq:lipschitz}
|| f(\textbf{x}) - f(\textbf{y}) || \le \beta ||\textbf{x} - \textbf{y} ||.
\end{equation}

\item Let $f$ be a $\beta$-smooth function on $\mathbb{R}^{n}$. Then, $\forall \textbf{x}, \textbf{y} \in \mathbb{R}^{n}$, 
\begin{equation}
f(\textbf{x}) - f(\textbf{y}) \le \nabla f(\textbf{y})^{T}(\textbf{x} - \textbf{y}) + \frac{\beta}{2} ||\textbf{x} - \textbf{y}||^{2}_{2}.
\end{equation}

\item A function $f$ satisfies the $\gamma$-Polyak-Lojasiewicz ($\gamma$-PL) condition~\cite{karimi2016linear}. Then, $\forall \textbf{x} \in \mathbb{R}^{n}$,
\begin{equation}
\label{eq:pl}
||\nabla f(\textbf{x}) ||_{2}^{2} \ge \gamma (f(\textbf{x}) - f^{*}).
\end{equation}
where $f^{*}$ refers to the optimal value over the input domain. The $\gamma$-PL condition is a significant property for a non-convex function $f$ because a global minimization can be attained from $\nabla f(\textbf{x}) = 0$, and a local minimum point corresponds to the global one. Furthermore, if a function is convex and also satisfies $\gamma$-PL condition, the function is strongly convex.

\item Jensen's inequality: Let $X$ be a random vector taking values in a non-empty convex set $K\subset \mathbb{R}^{n}$ with a finite expectation $\mathbb{E}[X]$, and $f$ be a measurable convex function defined over $K$. Then, $\mathbb{E}[X]$ is in $K$, $\mathbb{E}[f(X)]$ is finite, and the following inequality holds
\begin{equation}
f(\mathbb{E}[X]) \le \mathbb{E}[f(X)].
\end{equation}
\end{itemize}

\subsection{Empirical Rademacher Complexity}
Empirical Rademacher complexity~\cite{bartlett2002rademacher} is a measure of how well the function class correlates with the Rademacher random value. The references \cite{fan2019selective, zhu2009human, wainwright2019high} show that a function class with a larger empirical Rademacher complexity is more likely to be overfitted to the training data. 
\begin{definition}
\label{def1}
A Rademacher random variable takes on values $\pm 1$ and is defined by the uniform distribution as:
\begin{equation}
\label{eq:rrv}
\sigma_{i} = \begin{cases}
			$1$, \hspace{4mm} \text{with probability $\frac{1}{2}$} \\
			$-1$, \hspace{3mm} \text{with probability $\frac{1}{2}$}.
		\end{cases}
\end{equation}
\end{definition}

\begin{definition}
The empirical Rademacher complexity of a hypothesis space $\mathbb{H}$ of functions $h: \mathbb{R}^{n} \rightarrow \mathbb{R}$ with respect to $N$ samples $S = \{\textbf{x}_{1}, \textbf{x}_{2}, ..., \textbf{x}_{N}\}$ is: 
\begin{equation}
\hat{\mathcal{R}}_{S}(\mathbb{H}) := \mathbb{E}_{\boldsymbol{\sigma}} \left[\sup\limits_{h\in \mathbb{H}} \frac{1}{N} \sum\limits_{i=1}^{N} \sigma_{i} h(\textbf{x}_{i}) \right],
\end{equation}
where $\boldsymbol{\sigma} = \{\sigma_{1}, \sigma_{2}, ..., \sigma_{N}\}$ indicates a set of $N$ Rademacher random variables.
\end{definition}

\begin{lemma}[Talagrand's Lemma~\cite{mohri2018foundations}]
\label{lemm}
Let $\Phi_{1}, ..., \Phi_{N}$ be $L$-Lipschitz functions and $\sigma_{1}, ..., \sigma_{N}$ be Rademacher random variables. Then, for any hypothesis space $\mathbb{H}$ of functions $h: \mathbb{R}^{n} \rightarrow \mathbb{R}$ with respect to $N$ samples $S = \{\textbf{x}_{1}, \textbf{x}_{2}, ..., \textbf{x}_{N}\}$, the following inequality holds
\begin{equation}
\begin{split}
\frac{1}{N}\mathbb{E}_{\boldsymbol{\sigma}}\left[\sup\limits_{h\in \mathbb{H}} \sum\limits_{i=1}^{N} \sigma_{i}(\Phi_{i} \circ h)(\textbf{x}_{i})\right] &\le \frac{L}{N} \mathbb{E}_{\boldsymbol{\sigma}}\left[\sup\limits_{h\in \mathbb{H}}\sum\limits_{i=1}^{N} \sigma_{i} h(\textbf{x}_{i})	\right]  \\
&= L \hat{\mathcal{R}}_{S}(\mathbb{H}).
\end{split}
\end{equation}
\end{lemma}

\subsection{MAE and MSE}
\begin{definition}
\noindent MAE measures the average magnitude of absolute differences between $N$ predicted vectors $S = \{\textbf{x}_{1}, \textbf{x}_{2}, ..., \textbf{x}_{N}\}$ and $N$ actual observations $S^{*} = \{\textbf{y}_{1}, \textbf{y}_{2}, ..., \textbf{y}_{N}\}$, which is related to $L_{1}$ norm $(||\cdot||_{1})$ and the corresponding loss function is defined as:
\begin{equation}
    \label{eq:mae}
    \mathcal{L}_{MAE}(S, S^{*}) = \frac{1}{N} \sum\limits_{i=1}^{N}||\textbf{x}_{i} - \textbf{y}_{i}||_{1}. 
\end{equation}

\noindent Mean Squared Error (MSE)~\cite{paez1972minimum} denotes a quadratic scoring rule that measures the average magnitude of $N$ predicted vectors $S = \{\textbf{x}_{1}, \textbf{x}_{2}, ..., \textbf{x}_{N}\}$ and $N$ actual observations $S^{*} = \{\textbf{y}_{1}, \textbf{y}_{2}, ..., \textbf{y}_{N}\}$, which is related to $L_{2}$ norm $(||\cdot||_{2})$ and the corresponding loss function is shown as:
\begin{equation}
    \label{eq:mse}
    \mathcal{L}_{MSE}(S, S^{*}) =  \frac{1}{N} \sum\limits_{i=1}^{N} ||\textbf{x}_{i} - \textbf{y}_{i}||^{2}_{2}. 
\end{equation}
\end{definition}

\section{Upper Bounding MAE for DNN Based Vector-to-Vector Regression}
\label{sec3}
This section derives the upper bound on a generalized loss of MAE for DNN based vector-to-vector regression. We first discuss the error decomposition technique for MAE. Then, we upper bound each decomposed error, and attain an aggregated upper bound on MAE.

\subsection{Error Decomposition of MAE}

Based on the traditional error decomposition approach~\cite{mohri2018foundations, shalev2014understanding}, we generalize the technique to the DNN based vector-to-vector regression, where the smooth ReLU activation function, the regression loss functions, and their associated hypothesis space are separately defined in Definition~\ref{def1}.

\begin{definition}
\label{def1}
A smooth vector-to-vector regression function is defined as $f_{v}^{*}:\mathbb{R}^{d}\rightarrow \mathbb{R}^{q}$, and a family of DNN based vector-to-vector functions is represented as $\mathbb{F} = \{ f_{v}: \mathbb{R}^{d} \rightarrow \mathbb{R}^{q} \}$, where a smooth ReLU activation is given as:
\begin{equation}
\label{eq:relu}
g_{u}(x) = \lim\limits_{t\rightarrow \infty} \frac{\ln(1 + \exp(tx))}{t}.
\end{equation}
Moreover, we assume $\mathbb{L} = \{\mathcal{L}(f_{v}, f_{v}^{*}): \mathbb{R}^{d} \times \mathbb{R}^{d} \rightarrow \mathbb{R}, f_{v} \in \mathbb{F} \}$ as the family of generalized MAE loss functions. For simplicity, we denote $\mathcal{L}(f_{v}, f_{v}^{*})$ as $\mathcal{L}(f_{v})$. Besides, we denote $\mathcal{D}$ as a distribution over $\mathbb{R}^{d}$.
\end{definition}

The following proposition bridges the connection of Rademacher complexity between the family $\mathbb{L}$ of generalized MAE loss functions and the family $\mathbb{F}$ of DNN based vector-to-vector functions.

\begin{prop}
\label{prop:rad}
For any sample set $S = \{\textbf{x}_{1}, ..., \textbf{x}_{N}\}$ drawn i.i.d. according to a given distribution $\mathcal{D}$, the Rademacher complexity of the family $\mathbb{L}$ is upper bounded as:
\begin{equation}
\hat{\mathcal{R}}_{S}(\mathbb{L}) \le \mathcal{\hat{R}_{S}}(\mathbb{F}),
\end{equation}
where $\mathcal{\hat{R}}_{S}(\mathbb{F})$ denotes the empirical Rademacher complexity over the family $\mathbb{F}$, and it is defined as:
\begin{equation}
\label{eq:rsf}
\hat{\mathcal{R}}_{S}(\mathbb{F}) = \mathbb{E}_{\boldsymbol{\sigma}}\left[ \frac{1}{N}\sup\limits_{f_{v} \in \mathbb{F}} \sum\limits_{i=1}^{N} (\sigma_{i} \textbf{1})^{T} f_{v}(\textbf{x}_{i})\right].
\end{equation}
\end{prop}

\begin{proof}
We first show that MAE loss function is $1$-Lipschitz continuous. For two vectors $\textbf{y}_{1}, \textbf{y}_{2}\in \mathbb{R}^{q}$ and a fixed vector $\textbf{y} \in \mathbb{R}^{q}$, the MAE loss difference is 
\begin{equation}
\begin{split}
&\hspace{4mm} \left|\mathcal{L}(\textbf{y}_{1}, \textbf{y}) - \mathcal{L}(\textbf{y}_{2}, \textbf{y}) \right|	 \\
&= \left| ||\textbf{y}_{1} - \textbf{y} ||_{1}  -  ||\textbf{y}_{2} - \textbf{y} ||_{1} \right|	 \\
&\le  ||\textbf{y}_{1} - \textbf{y}_{2}||_{1}      	      \hspace{16mm} (\text{triangle inequality}).
\end{split}
\end{equation}
Since the target function $f_{v}^{*}$ is given, $\mathcal{L}(f_{v}) \in \mathbb{L}$ is $1$-Lipschitz. By applying Lemma~\ref{lemm}, we obtain that
\begin{equation}
\label{eq:l3}
\begin{split}
\hat{\mathcal{R}}_{S}(\mathbb{L}) &= \frac{1}{N} \mathbb{E}_{\boldsymbol{\sigma}}\left[ \sup\limits_{f_{v}\in \mathbb{F}} \sum\limits_{i=1}^{N} \sigma_{i} \mathcal{L}(f_{v}(\textbf{x}_{i}))  \right]	\\
&= \frac{1}{N}  \mathbb{E}_{\boldsymbol{\sigma}}\left[ \sup\limits_{f_{v}\in \mathbb{F}} \sum\limits_{i=1}^{N} \sigma_{i} \mathcal{L}( \sum\limits_{m=1}^{q} \langle \textbf{1}_{m}, f_{v}(\textbf{x}_{i})\rangle \textbf{1}_{m})  \right] \\
&\le \frac{1}{N} \mathbb{E}_{\boldsymbol{\sigma}}\left[ \sup\limits_{f_{v}\in \mathbb{F}} \sum\limits_{i=1}^{N} (\sigma_{i} \textbf{1})^{T} f_{v}(\textbf{x}_{i}) \right] = \hat{\mathcal{R}}_{S}(\mathbb{F}).
\end{split}
\end{equation}
\end{proof}

Since $\mathcal{\hat{R}}_{S}({\mathbb{F}})$ is an upper bound of $\mathcal{\hat{R}}_{S}({\mathbb{L}})$, we can utilize the upper bound on $\mathcal{\hat{R}}_{S}({\mathbb{L}})$ to derive the upper bound for $\mathcal{\hat{R}}_{S}({\mathbb{F}})$. Next, we adopt the error decomposition technique to attain an aggregated upper bound which consists of three error components. 

\begin{theorem}
\label{lemma1}
Let $\mathcal{\hat{L}}\in \mathbb{L}$ denote the loss function for a set of samples $S$ drawn i.i.d. according to a given distribution $\mathcal{D}$, and define $\hat{f}_{v}\in \mathbb{F}$ as an ERM for $\mathcal{\hat{L}}$. For a generalized MAE loss function $\mathcal{L} \in \mathbb{L}$, $\epsilon > 0$, and $0 < \delta < 1$, there exists $f_{v}^{\epsilon} \in \mathbb{F}$ such that $\mathcal{L}(f_{v}^{\epsilon}) \le \inf_{f_{v}\in \mathbb{F}}\mathcal{L}(f_{v}) + \epsilon$. Then, with a probability of $\delta$, we attain that 
\begin{equation}
\label{eq:upper3}
\begin{split}
&\hspace{5mm} \mathcal{L}(\hat{f}_{v}) \\
&\le \underbrace{\inf\limits_{f_{v}\in \mathbb{F}} \mathcal{L}(f_{v})}_{\text{Approx. error}} + \underbrace{2\sup\limits_{f_{v}\in \mathbb{F}} |\mathcal{L}(f_{v}) - \hat{\mathcal{L}}(f_{v})|}_{\text{Estimation error}} + \underbrace{\mathcal{L}(f_{v}^{\epsilon}) - \inf\limits_{f_{v}\in \mathbb{F}}\mathcal{L}(f_{v})}_{\text{Optimization error}}	\\
&\le \inf\limits_{f_{v}\in \mathbb{F}} \mathcal{L}(f_{v}) + 2\hat{\mathcal{R}}_{S}(\mathbb{F}) + \epsilon.	\\
\end{split}
\end{equation}
\end{theorem}

\begin{proof}
\begin{equation*}
\begin{split}
\mathcal{L}(\hat{f}_{v}) &=  \inf\limits_{f_{v}\in \mathbb{F}}\mathcal{L}(f_{v}) +  \mathcal{L}(\hat{f_{v}}) - \mathcal{L}(f_{v}^{\epsilon}) + \mathcal{L}(f_{v}^{\epsilon}) - \inf\limits_{f_{v}\in \mathbb{F}}\mathcal{L}(f_{v})\\
&\le \inf\limits_{f_{v}\in \mathbb{F}}\mathcal{L}(f_{v}) + \mathcal{L}(\hat{f}_{v}) - \mathcal{L}(f_{v}^{\epsilon}) + \epsilon \\
&\le \inf\limits_{f_{v}\in \mathbb{F}}\mathcal{L}(f_{v}) + \mathcal{L}(\hat{f}_{v}) - \mathcal{\hat{L}}(\hat{f}_{v}) + \mathcal{\hat{L}}(f_{v}^{\epsilon}) - \mathcal{L}(f_{v}^{\epsilon}) + \epsilon \\
&\le \inf\limits_{f_{v}\in \mathbb{F}}\mathcal{L}(f_{v}) + 2\sup\limits_{f_{v}\in \mathbb{F}} |\mathcal{L}(f_{v}) - \mathcal{\hat{L}}(f_{v}) | + \epsilon.
\end{split}
\end{equation*}

Then, we continue to upper bound the term $2\sup_{f_{v }\in \mathbb{F}} |\mathcal{L}(f_{v}) - \mathcal{\hat{L}}(f_{v}) |$. We first define $\mu$ as the expected value of $\sup_{f_{v} \in \mathbb{F}} |\mathcal{L}(f_{v}) - \mathcal{\hat{L}}(f_{v}) |$, and then introduce the fact that
\begin{equation}
\mu = \mathbb{E}\left[\sup\limits_{f_{v}\in \mathbb{F}} \left|\mathcal{L}(f_{v}) - \mathcal{\hat{L}}(f_{v}) \right|\right]  \le 2 \mathcal{\hat{R}}_{S}(\mathbb{L}),
\end{equation}
which is justified by Lemma~\ref{lemma3} in Appendix~\ref{sec:appA}.
Then, for a small $\delta$ $(0 < \delta < 1)$, we apply the Hoeffding's bound~\cite{hoeffding1994probability} as follows
\begin{equation*}
\begin{split}
&\text{P} \left( 2 \sup\limits_{f_{v}\in \mathbb{F}} \left| \mathcal{L}(f_{v}) - \mathcal{\hat{L}}(f_{v}) \right| \le \nu \right) \ge 1 - 2\exp\left( -2N (\nu - \mu)^{2} \right)	\\
&\hspace{36mm} \ge 1- 2\exp\left( -2N (\nu - 2\mathcal{\hat{R}}_{S}(\mathbb{L}))^{2} \right)	\\
&\hspace{36mm} = \delta,
\end{split}
\end{equation*}
which can derive $\nu$ as:
\begin{equation*}
\nu =  2 \mathcal{\hat{R}}_{S}(\mathbb{L}) + \sqrt{\frac{1}{2N} \ln\left( \frac{2}{1 - \delta} \right)},
\end{equation*}
and we thus obtain that 
\begin{equation*}
\begin{split}
2 \sup\limits_{f_{v}\in \mathbb{F}} \left| \mathcal{L}(f_{v}) - \mathcal{\hat{L}}(f_{v}) \right| &\le 2 \mathcal{\hat{R}}_{S}(\mathbb{L}) + \sqrt{\frac{1}{2N} \ln\left( \frac{2}{1 - \delta} \right)}.
\end{split}
\end{equation*}
Therefore,
\begin{equation*}
\begin{split}
\mathcal{L}(\hat{f}_{v}) &\le  \inf\limits_{f_{v}\in \mathbb{F}}\mathcal{L}(f_{v})  + \left(2\mathcal{\hat{R}}_{S}(\mathbb{L}) + \sqrt{\frac{1}{2N} \ln(\frac{2}{1 - \delta})} \right) + \epsilon \\
&\le  \inf_{f_{v}\in \mathbb{F}}\mathcal{L}(f_{v}) +  2\mathcal{\hat{R}}_{S}(\mathbb{F}) +  \sqrt{\frac{1}{2N} \ln(\frac{2}{1 - \delta})} + \epsilon 	\\
&\approx  \inf\limits_{f_{v}\in \mathbb{F}}\mathcal{L}(f_{v}) + 2\mathcal{\hat{R}}_{S}(\mathbb{F}) + \epsilon \hspace{2mm} \text{(for sufficiently large $N$). }
\end{split}
\end{equation*}
\end{proof}

Next, the remainder of this section presents how to upper bound the approximation error, approximation error, and optimization error, respectively.

\subsection{An Upper Bound for Approximation Error}
The upper bound for the approximation error is shown in Theorem~\ref{thm:thm1}, which is based on the modification of our previous theorem for the representation power of DNN based vector-to-vector regression~\cite{qi2019theory}. 

\begin{theorem}
\label{thm:thm1}
For a smooth vector-to-vector regression target function $f_{v}^{*}: \mathbb{R}^{d} \rightarrow \mathbb{R}^{q}$, there exists a DNN $\bar{f}_{v}\in \mathbb{F}$ with $k (k \ge 2)$ modified smooth ReLU based hidden layers, where the width of each hidden layer is at least $d+2$ and the top hidden layer has $n_{k} (n_{k} \ge d+2)$ units. Then, we derive the upper bound for the approximation error as:
\begin{equation}
\label{eq:con1}
\begin{split}
\inf\limits_{f_{v}\in \mathbb{F}} \mathcal{L}(f_{v}) = ||f_{v}^{*} - \bar{f}_{v} ||_{1} = \mathcal{O}\left(\frac{q}{(n_{k} + k - 1)^{\frac{r}{d}}}\right),
\end{split}
\end{equation}
where a smooth ReLU function is defined in Eq.~(\ref{eq:relu}), and $r$ refers to the differential order of $f_{v}$.
\end{theorem}

The smooth ReLU function in Eq.~(\ref{eq:relu}) is essential to derive the upper bound for the optimization error. Since Theorem~\ref{thm:thm1} is a direct result from Lemma 2 in \cite{hanin2019universal} where the standard ReLU is employed and does not consider Barron's bound for activation functions~\cite{barron}, the smooth ReLU function can be flexibly utilized in Theorem~\ref{thm:thm1} because it is a close approximation to the standard ReLU function. Moreover, Theorem~\ref{thm:thm1} requires at least $d+2$ neurons for a $d$-dimensional input vector to achieve the upper bound.

\subsection{An Upper Bound for Estimation Error}
Since the estimation error in Eq.~(\ref{eq:upper3}) is upper bounded by the empirical Rademacher complexity $\hat{\mathcal{R}}_{S}(\mathbb{F})$, we derive Theorem~\ref{thm:thm2} to present an upper bound on $\hat{\mathcal{R}}_{S}(\mathbb{F})$. The derived upper bound is explicitly controlled by the constraints of weights in the hidden layers, inputs, and the number of training data. In particular, the constraint of $L_{1}$ norm is set to the top hidden layer, and $L_{2}$ norm is imposed on the other hidden layers. 

\begin{theorem}
\label{thm:thm2}
For a DNN based vector-to-vector mapping function $f_{v}(\textbf{x}) = \textbf{W}_{k} \circ g_{u} \circ \textbf{W}_{k-1} \circ \cdot\cdot\cdot \circ \textbf{W}_{2}\circ g_{u} \circ \textbf{W}_{1}(\textbf{x}): \mathbb{R}^{d} \rightarrow \mathbb{R}^{q}$ with a smooth ReLU function $g_{u}$ as in Eq.~(\ref{eq:relu}) and $\forall i\in [k]$, $\textbf{W}_{i}$ being the weight matrix of the $i$-th hidden layer, we obtain an upper bound for the empirical Rademacher complexity $\hat{\mathcal{R}}_{S}(\mathbb{F})$ with regularized constraints of the weights in each hidden layer, and the $L_{2}$ norm of input vectors $\textbf{x}$ is bounded by $s$.
\begin{equation}
\begin{split}
\label{eq:upper2}
&\hspace{0.5mm} 2\sup_{f_{v}\in \mathbb{F}}|\mathcal{L}(f_{v}) - \mathcal{L}(\hat{f}_{v})| \le 2\hat{\mathcal{R}}_{S}(\mathbb{F}) \le \frac{2 q \Lambda^{'}\Lambda^{k-1}s}{\sqrt{N}} 	\\
&s.t., \hspace{5mm} ||\textbf{W}_{k}(i, :)||_{1} \le \Lambda^{'}, \forall i \in [q] \\
&\hspace{11.5mm} ||\textbf{W}_{j}(a, :)||_{2} \le \Lambda, \forall j\in [k-1], a\in [n_{j}] \\
&\hspace{11.5mm} ||\textbf{x}||_{2} \le s,
\end{split}
\end{equation}
where $\textbf{W}_{j}(m, n)$ is an element associated with the $j$-th hidden layer of DNN where $m$ is indexed to neurons in the $j$-th hidden layer and $n$ is pointed to units of the $(j-1)$-th hidden layer, and $\textbf{W}_{j}(m, :)$ contains all weights from the $m$-th neuron to all units in the ($j-1$)-th hidden layer.
\end{theorem}
\begin{proof}
We first consider an ANN with one hidden layer of $n$ neuron units with the smooth ReLU function $g_{u}$ as Eq.~(\ref{eq:relu}), and also denote $\mathbb{\hat{F}}$ as a family of ANN based vector-to-vector regression functions. $\mathbb{\hat{F}}$ can be decompoed into the sum of $q$ subspaces $\sum_{i=1}^{q}\mathbb{\hat{F}}_{i}$ and each subspace $\mathbb{\hat{F}}_{m}$ is defined as:
\begin{equation*}
\label{eq:hi}
\mathbb{\hat{F}}_{m} = \left\{ \textbf{x} \rightarrow \sum\limits_{j=1}^{n}w_{j}g_{u}(\textbf{u}_{j}^{T} \textbf{x})\cdot \textbf{1}_{m}: ||\textbf{w}||_{1} \le \Lambda^{'}, ||\textbf{u}_{j}||_{2} \le \Lambda \right\}, 
\end{equation*}
where $n$ is the number of hidden neurons, $\forall j\in [n]$, $\textbf{w}$ and $\textbf{u}_{j}$ separately correspond to $\textbf{W}_{2}(m, :)$ and $\textbf{W}_{1}(j, :)$ in Eq.~(\ref{eq:upper2}). Given $N$ data samples $\{\textbf{x}_{1},  \textbf{x}_{2}, ..., \textbf{x}_{N}\}$, the empirical Rademacher complexity of $\mathbb{\hat{F}}_{m}$ is bounded as:
\begin{equation}
\begin{split}
\label{eq:rs}
 \hat{\mathcal{R}}&_{S}(\mathbb{\hat{F}}_{m}) = \frac{1}{N} \mathbb{E}_{\boldsymbol{\sigma}}\left[\sup\limits_{||\textbf{w}||_{1}\le \Lambda^{'}, ||\textbf{u}_{j}||_{2}\le\Lambda} \sum\limits_{i=1}^{N} \sigma_{i} \sum\limits_{j=1}^{n}w_{j} g_{u}(\textbf{u}_{j}^{T}\textbf{x}_{i})\right]  \\
&= \frac{1}{N} \mathbb{E}_{\boldsymbol{\sigma}} \left[\sup\limits_{||\textbf{w}||_{1}\le \Lambda^{'}, ||\textbf{u}_{j}||_{2}\le\Lambda} \sum\limits_{j=1}^{n} w_{j} \sum\limits_{i=1}^{N}\sigma_{i} g_{u}(\textbf{u}_{j}^{T} \textbf{x}_{i}) \right]  \\
&\le \frac{\Lambda^{'}}{N} \mathbb{E}_{\boldsymbol{\sigma}} \left[\sup\limits_{||\textbf{u}_{j}||_{2} \le \Lambda} \max\limits_{j\in [n]} \left|\sum\limits_{i=1}^{N} \sigma_{i} g_{u}(\textbf{u}_{j}^{T} \textbf{x}_{i})\right|\right] \hspace{1mm} \text{(H\"{o}lder's ineq.)} \\
&= \frac{\Lambda^{'}}{N} \mathbb{E}_{\boldsymbol{\sigma}}\left[{\sup\limits_{||\textbf{u}||_{2} \le \Lambda}} \left| \sum\limits_{i=1}^{N} \sigma_{i} g_{u}(\textbf{u}^{T} \textbf{x}_{i}) \right| \right]. \\
&\le \frac{\Lambda^{'}}{N} \mathbb{E}_{\boldsymbol{\sigma}} \left[ \sup\limits_{||\textbf{u}||_{2} \le \Lambda} \left| \sum\limits_{i=1}^{N}\sigma_{i} \textbf{u}^{T}\textbf{x}_{i}\right|\right] \hspace{15mm} \text{(c.f. Lemma~\ref{lemm})}  \\
&\le \frac{\Lambda \Lambda^{'}}{N} \mathbb{E}_{\boldsymbol{\sigma}}\left[|| \sum\limits_{i=1}^{N} \sigma_{i} \textbf{x}_{i} ||_{2}\right] \hspace{10mm} \text{(Cauchy-Schwartz ineq.)} \\
&\le \frac{\Lambda \Lambda^{'}}{N} \sqrt{\mathbb{E}_{\boldsymbol{\sigma}}\left[ || \sum\limits_{i=1}^{N} \sigma_{i} \textbf{x}_{i} ||_{2}^{2}\right]} \hspace{12mm} \text{(Jensen's inequality)}.	\\
\end{split}
\end{equation}
The last term in the inequality~(\ref{eq:rs}) can be further simplified based on the independence of $\sigma_{i}$s. Thus, we finally derive the upper bound as:
\begin{equation}
\label{eq:rrs2}
\begin{split}
\hat{\mathcal{R}}_{S}(\mathbb{\hat{F}}_{m}) &\le \frac{\Lambda \Lambda^{'}}{N} \sqrt{\mathbb{E}_{\boldsymbol{\sigma}}\left[|| \sum\limits_{i=1}^{N} \sigma_{i} \textbf{x}_{i} ||_{2}^{2}\right]}	\\
&= \frac{\Lambda\Lambda^{'}}{N} \sqrt{\sum\limits_{i, j=1}^{N} \mathbb{E}_{\boldsymbol{\sigma}} [\sigma_{i}\sigma_{j}](\textbf{x}_{i}^{T} \textbf{x}_{j})} \\
&= \frac{\Lambda\Lambda^{'}}{N} \sqrt{\sum\limits_{i=1}^{N} ||\textbf{x}_{i}||_{2}^{2}} \hspace{14mm} \text{(independence of $\sigma_{i}$s)} 	\\
&\le \frac{\Lambda\Lambda^{'}s}{\sqrt{N}}.
\end{split}
\end{equation}

The upper bound for $\mathcal{\hat{R}}_{S}(\mathbb{\hat{F}})$ is derived based on the fact that for $q$ families of functions $\mathbb{\hat{F}}_{m}, m\in[q]$, there is $\mathcal{\hat{R}}_{S}(\mathbb{F}) = \hat{\mathcal{R}}_{S}(\sum_{m=1}^{q}\mathbb{\hat{F}}_{m}) = \sum_{m=1}^{q} \hat{\mathcal{R}}_{S}(\mathbb{\hat{F}}_{m})$, and thus
\begin{equation}
\label{eq:upper4}
\hat{\mathcal{R}}_{S}(\mathbb{\hat{F}}) = \sum\limits_{m=1}^{q} \mathcal{\hat{R}}_{S}(\mathbb{\hat{F}}_{m}) \le \frac{q \Lambda\Lambda^{'}s}{\sqrt{N}},
\end{equation}
which is an extension of the empirical Rademacher identities~\cite{mohri2018foundations}, which is demonstrated in Lemma~\ref{lemma2} of Appendix~\ref{sec:appA}.

Then, for the family of DNNs $\mathbb{F}$ with $k$ hidden layers activated by the smooth ReLU function, we iteratively apply Lemma~\ref{lemm} and end up attaining the upper bound as: 
\begin{equation*}
\label{eq:final1}
\begin{split}
&\hspace{4mm} \hat{\mathcal{R}}_{S}(\mathbb{F}) 	\\
&= \mathbb{E}_{\boldsymbol{\sigma}}\left[ \sup\limits_{\forall l, w_{j_{l}} \in \mathbb{U}}\sum\limits_{m=1}^{q} \sum\limits_{i=1}^{N}\sigma_{i}\sum\limits_{j_{k}=1}^{n_{k}}w_{j_{k}} g_{u}( \cdot\cdot\cdot \sum\limits_{j_{1}=1}^{n_{1}}w_{j_{1}}g_{u}(\textbf{u}_{j}^{T} \textbf{x}_{i}))\right]	\\
&\le \mathbb{E}_{\boldsymbol{\sigma}}\left[\sup\limits_{\forall l, w_{j_{l}} \in \mathbb{U}} \sum\limits_{m=1}^{q} \sum\limits_{i=1}^{N}\sigma_{i}\sum\limits_{j_{k}=1}^{n_{k}}w_{j_{k}} \cdot\cdot\cdot \sum\limits_{j_{1}=1}^{n_{1}}w_{j_{1}}\textbf{u}_{j}^{T} \textbf{x}_{i} \right]		\\
&\le \frac{q\Lambda^{'} \Lambda^{k-1} s}{\sqrt{N}}, 
\end{split}
\end{equation*}
where $w_{j_{1}},..., w_{j_{k}}$ are selected from the hypothesis space $\mathbb{U} = \left\{ w_{j_{1}}, ..., w_{j_{k}}: \sum\limits_{j_{k}=1}^{n_{k}}|w_{j_{k}}| \le \Lambda^{'},  \sqrt{ \sum\limits_{j_{i}=1}^{n_{i}}w_{j_{i}}^{2} } \le \Lambda, \forall i\in [k-1] \right\}$.
\end{proof}

\subsection{An Upper Bound for Optimization Error}
Next, we derive an upper bound for the optimization error. A recent work~\cite{bassily2018exponential} has shown that the $\gamma$-PL property can be ensured if neural networks are configured with the setup of the ``over-parametrization''~\cite{vaswani2018fast}, which is induced from the two facts as follows:

\begin{itemize}
\item Neural networks can satisfy $\gamma$-PL condition, when the weights of hidden layers are initialized near the global minimum point~\cite{vaswani2018fast, charles2017stability}.

\item As the neural network involves more parameters, the update of parameters moves less, and there exists a global minimum point near the random initialization~\cite{allen2018convergence, li2018learning}.
\end{itemize}

Thus, the upper bound on the optimization error can be tractably derived in the context of the $\gamma$-PL condition for the generalized MAE loss $\mathcal{L}(\cdot)\in \mathbb{L}$. Since the smooth ReLU function admits smooth DNN based vector-to-vector functions, which can lead to an upper bound on the optimization error as:
\begin{equation}
\label{eq:upper5}
\epsilon = \mathcal{L}(f_{v}^{\epsilon}) - \inf\limits_{f_{v}\in \mathbb{F}} \mathcal{L}(f_{v}) \le \frac{\mu M^{2} \beta}{2\gamma}.
\end{equation}

To achieve the upper bound in Eq.~(\ref{eq:upper5}), we assume that the stochastic gradient descent (SGD) algorithm can result in an approximately equal optimization error for both the generalized MAE loss $\mathcal{L}(\cdot) \in \mathbb{L}$ and the empirical MAE loss $\mathcal{\hat{L}}(\cdot) \in \mathbb{L}$.

 More specifically, for two DNN based vector-to-vector regression functions $f_{v}^{\epsilon}\in \mathbb{H}$ and $\hat{f}_{v}^{\epsilon} \in \mathbb{H}$, we have that
\begin{equation}
\begin{split}
\epsilon &= \mathcal{L}(f_{v}^{\epsilon}) - \inf\limits_{f_{v}\in \mathbb{F}} \mathcal{L}(f_{v}) \approx  \mathcal{\hat{L}}(\hat{f}_{v}^{\epsilon}) - \inf\limits_{f_{v}\in \mathbb{F}}\mathcal{\hat{L}}(f_{v}).
 \end{split}
\end{equation}

Thus, we focus on analyzing $\mathcal{\hat{L}}(f_{v})$ because it can be updated during the training process. We assume that $\mathcal{\hat{L}}(f_{v})$ is $\beta$-smooth with $||\nabla \mathcal{\hat{L}}(f_{v})||_{2} \le M$ and it also satisfies the $\gamma$-PL condition from an early iteration $t_{0}$. Besides, the learning rate of SGD is set to $\mu$.

Moreover, we define $f_{v, \textbf{w}_{t}} \in \mathbb{F}$ as the function with an updated parameter $\textbf{w}_{t}$ at the iteration $t$, and denote $f_{v,\textbf{w}_{*}} \in \mathbb{F}$ as the function with the optimal parameter $\textbf{w}_{*}$.  The smoothness of $\mathcal{\hat{L}}(\cdot)$ implies that
\begin{equation}
\label{eq:smoothness}
\begin{split}
&\hspace{4mm} \mathcal{\hat{L}}(f_{v, \textbf{w}_{t+1}}) - \mathcal{\hat{L}}(f_{v, \textbf{w}_{t}}) - \langle \nabla \mathcal{\hat{L}}(f_{v, \textbf{w}_{t}}), \textbf{w}_{t+1} - \textbf{w}_{t} \rangle     \\
&\le \frac{\beta}{2} || \textbf{w}_{t} - \textbf{w}_{t+1} ||_{2}^{2}.
\end{split}
\end{equation}

Then, we apply the SGD algorithm to update model parameters at the iteration $t$ as:
\begin{equation}
\label{eq:sgd}
\textbf{w}_{t+1} = \textbf{w}_{t} - \mu \nabla \mathcal{\hat{L}}(f_{v, \textbf{w}_{t}}).
\end{equation}

Next, we substitute $-\mu \nabla \mathcal{\hat{L}}(f_{v, \textbf{w}_{t}})$ in Eq.~(\ref{eq:sgd}) for $\textbf{w}_{t+1} - \textbf{w}_{t}$ in the inequality~(\ref{eq:smoothness}), we have that
\begin{equation}
\label{eq:smooth1}
\begin{split}
& \hspace{4mm} \mathcal{\hat{L}}(f_{v, \textbf{w}_{t+1}}) - \mathcal{\hat{L}}(f_{v, \textbf{w}_{t}}) + \mu ||\nabla \mathcal{\hat{L}}(f_{v, \textbf{w}_{t}})||_{2}^{2}     \\
& \le \frac{\beta \mu^{2}}{2} ||\nabla \mathcal{\hat{L}}(f_{v, \textbf{w}_{t}})||_{2}^{2}.
\end{split}
\end{equation}

By employing the condition $||\nabla \mathcal{\hat{L}}(f_{v, \textbf{w}_{t}})||_{2}^{2} \le M^{2}$, we further derive that
\begin{equation}
\label{eq:smooth2}
 \mathcal{\hat{L}}(f_{v, \textbf{w}_{t+1}}) - \mathcal{\hat{L}}(f_{v, \textbf{w}_{t}}) + \mu ||\nabla \mathcal{\hat{L}}(f_{v, \textbf{w}_{t}})||_{2}^{2} \le \frac{\mu^{2} M^{2}\beta}{2}.
\end{equation}

Furthermore, we employ the $\gamma$-PL condition to Eq.~(\ref{eq:smooth2}) and obtain the inequalities as:
\begin{equation}
\label{eq:smooth3}
\begin{split}
&\hspace{4mm} \mathcal{\hat{L}}(f_{v, \textbf{w}_{t+1}}) - \mathcal{\hat{L}}(f_{v, \textbf{w}_{*}})			\\
&\le \left(\mathcal{\hat{L}}(f_{v, \textbf{w}_{t}}) - \mathcal{\hat{L}}(f_{v, \textbf{w}_{*}}) - \gamma\mu (\mathcal{\hat{L}}(f_{v, \textbf{w}_{t}}) - \mathcal{\hat{L}}(f_{v, \textbf{w}_{*}}))\right)  \\
& \hspace{72mm} +  \frac{\mu^{2} M^{2}\beta}{2} \\
&\le (1 - \mu\gamma) \left(\mathcal{\hat{L}}(f_{v, \textbf{w}_{t}}) - \mathcal{\hat{L}}(f_{v, \textbf{w}_{*}})\right) + \frac{\mu^{2} M^{2}\beta}{2}  \\
&\le (1 - \mu\gamma)^{2} \left(\mathcal{\hat{L}}(f_{v, \textbf{w}_{t-1}}) - \mathcal{\hat{L}}(f_{v, \textbf{w}_{*}})\right) + \sum\limits_{i=0}^{1} (1 - \gamma \mu)^{i} \frac{\mu^{2} M^{2}\beta}{2}	\\
&\le \cdot\cdot\cdot 	\\
&\le (1 - \mu\gamma)^{t-t_{0}+1} \left(\mathcal{\hat{L}}(f_{v, \textbf{w}_{t_{0}}}) - \mathcal{\hat{L}}(f_{v, \textbf{w}_{*}})\right) + \\
& \hspace{56mm} \sum\limits_{i=0}^{t-t_{0}} (1 - \gamma \mu)^{i} \frac{\mu^{2} M^{2}\beta}{2}	\\
&\le (1 - \mu\gamma)^{t-t_{0}+1} \left(\mathcal{\hat{L}}(f_{v, \textbf{w}_{t_{0}}}) - \mathcal{\hat{L}}(f_{v, \textbf{w}_{*}})\right) + \frac{\mu M^{2}\beta}{2\gamma}	\\
&\le \exp(-\mu\gamma(t-t_{0} + 1)) \left(\mathcal{\hat{L}}(f_{v, \textbf{w}_{t_{0}}}) - \mathcal{\hat{L}}(f_{v, \textbf{w}_{*}})\right) +  \frac{\mu M^{2}\beta}{2\gamma}.
\end{split}
\end{equation}

By connecting the optimization error in Eq.~(\ref{eq:upper5}) to our derived Eq.~(\ref{eq:smooth3}), we further have that
\begin{equation}
\label{eq:final}
\begin{split}
\epsilon &= \mathcal{L}(f_{v}^{\epsilon}) - \inf\limits_{f_{v}\in \mathbb{F}} \mathcal{L}(f_{v})	\\
&\approx \hat{\mathcal{L}}(f_{v}^{\epsilon}) - \inf\limits_{f_{v}\in \mathbb{F}}\hat{\mathcal{L}}(f_{v}) \\
&\le \exp(-\mu\gamma(T+1)) \left(\hat{\mathcal{L}}(f_{v, \textbf{w}_{0}}) -  \inf\limits_{f_{v}\in \mathbb{F}}\hat{\mathcal{L}}(f_{v})\right)+ \frac{\mu M^{2}\beta}{2\gamma} \\
&\approx \frac{\mu M^{2}\beta}{2\gamma},
\end{split}
\end{equation}
 where $T = t - t_{0}$ and $f_{v, \textbf{w}_{0}} \in \mathbb{F}$ denotes a function with an initial parameter $\textbf{w}_{0}$.  The inequality in Eq.~(\ref{eq:final}) suggests that when the number of iterations $T$ is sufficiently large, we eventually attain the upper bound as Eq.~(\ref{eq:upper5}).

\vspace{2mm}
\noindent \textbf{Remark 1:} The ``over-parametrization'' condition becomes difficult to be configured in practice when large datasets have to be dealt with. In such cases, the upper bound on the optimization error cannot be always guaranteed, but we can relax the configuration of ``over-parametrization'' for DNNs and assume the $\gamma$-PL condition to derive the upper bound on the optimization error. In doing so, our proposed upper bound can be applied to more general DNN based vector-to-vector regression functions. 

\subsection{An Aggregated Bound for MAE}
Based on the upper bounds for the approximation error, estimation error and optimization error, we can derive an upper bound for $\mathcal{L}(h_{S}^{ERM})$. Besides, the constraints as shown in Eq.~(\ref{eq:core_thm}), which arise from the upper bounds on the approximation, estimation and optimization errors, are necessary conditions to derive the upper bound (with a probability $\delta \in (0, 1)$) as:

\begin{equation}
\begin{split}
\label{eq:core_thm}
\mathcal{L}(\hat{f}_{v}&) \le \inf\limits_{f_{v}\in \mathbb{F}} \mathcal{L}(f_{v}) + 2\hat{\mathcal{R}}_{S}(\mathbb{F}) + \epsilon \\
&\le \mathcal{O}\left(\frac{q}{(n_{k} + k - 1)^{\frac{r}{d}}}\right) + \frac{2q\Lambda^{'}\Lambda^{k-1} s}{\sqrt{N}} + \frac{\mu M^{2}\beta}{2\gamma} \\
\text{s.t.,} \hspace{2mm} &\text{Smooth ReLU:} \lim\limits_{t\rightarrow +\infty}\frac{1}{t} \ln(1 + \exp(tx))		\\
&\text{Hidden Layers:}\hspace{2mm} n_{j} \ge d + 2, \forall j\in [k] \\
&\text{Regularization:} \hspace{1mm}  ||\textbf{W}_{k}(i, :)||_{1} \le \Lambda^{'}, \forall i \in [q] \\
&\hspace{22.5mm} ||\textbf{W}_{j}(m, :)||_{2} \le \Lambda, \forall j\in [k-1], m\in [n_{j}] \\
&\text{Bounded Inputs:} \hspace{1.5mm} ||\textbf{x}||_{2} \le s \\
&\text{Over-parametrization: \text{The number of parameters exceeds }}\\
&\hspace{33mm} \text{the amount of training data. }
\end{split}
\end{equation}

Eq.~(\ref{eq:core_thm}) suggests that several hyper-parameters are required to derive the upper bound, which makes it difficult to be utilized in practice because the prior setup of $\mu$, $M$, $\beta$ and $\gamma$ are strong assumptions in use. Section~\ref{sec:mse} discusses how to estimate MAE values of ANN or DNN based vector-to-vector regression in practical experiments. Besides, the term $\frac{q\Lambda^{'}\Lambda^{k-1}s}{\sqrt{N}}$ in Eq.~(\ref{eq:core_thm}) may become arbitrarily large when a large $k$ and $\Lambda > 1$ are concerned. Thus, we set $\Lambda$ as $1$ to ensure normalized weights of the first $k-1$ layers, and the amount of training data $N$ could be large enough to ensure a small estimation error.

The configuration of ``over-parametrization'' requires that the number of model parameters exceeds the amount of training data such that the $\gamma$-PL condition can be guaranteed and consequently the upper bound on the optimization error can be attained. However, when the setup of ``over-parametrization'' cannot be strictly satisfied, the $\gamma$-PL condition does not always hold. Then, we can still assume the $\gamma$-PL condition to obtain the upper bound~(\ref{eq:upper5}), which allows the derived upper bound applicable for more general DNN based vector-to-vector regression functions. 

\vspace{2mm}
\noindent \textbf{Remark 2:} Our work employs MAE as the loss function instead of MSE for the following reasons: (i) MSE does not satisfy the Lipschitz continuity such that the inequality Eq.~(\ref{eq:l3}) cannot be guaranteed~\cite{mae_spl}; (ii) The MAE loss function for vector-to-vector regression tasks can achieve better performance than MSE in experiments~\cite{willmott2005advantages}.

\section{Estimation of the MAE Upper Bounds}
\label{sec:mse}

MAE can be employed as the loss function for training an ANN or DNN based vector-to-vector regression function. In this section, we discuss how to make use of the theorems in Section~\ref{sec3} to estimate MAE upper bounds for the vector-to-vector regression models in our experiments.

Proposition~\ref{prop:prop2} provides an upper bound on MAE based on our theorem in Eq.~(\ref{eq:core_thm}), where $c$ and $b$ are two non-negative hyper-parameters to be estimated from the experimental MAE losses of the ANN based vector-to-vector regression. An ANN with the smooth ReLU activation function in Eq.~(\ref{eq:relu}) is a convex and smooth function, which implies that the local optimum point returned by the SGD algorithm corresponds to a global one. Then, the estimated hyper-parameters $c$ and $b$ can be used to estimate the MAE values of DNN-based vector-to-vector regression. 

\begin{prop}
\label{prop:prop2}
For a smooth target function $\hat{f}_{v}:\mathbb{R}^{d} \rightarrow \mathbb{R}^{q}$, we use $N$ training data samples to obtain a DNN based vector-to-vector regression function $f_{v} \in \mathbb{F}$ with $k$ smooth ReLU based hidden layers ($k\ge 2$), where the width of each hidden layer is at least $d+2$. Then, we can derive an upper bound for MAE as:
\begin{equation}
\label{eq:dnn_mse}
MAE(\hat{f}_{v}, f_{v}) \le \frac{cq}{(n_{k} + k - 1)^{\frac{r}{d}}} + \frac{2 q\Lambda^{'} \Lambda^{k-1}s}{\sqrt{N}} + b,
\end{equation}
where the hyper-parameters $c$ and $b$ are separately set as:

\begin{equation}
\label{eq:c}
c = \frac{(MAE_{1} - MAE_{2})l_{1}^{r/d}l_{2}^{r/d}}{q(l_{2}^{r/d} - l_{1}^{r/d})}, 
\end{equation}
\noindent and
\begin{equation}
\label{eq:b}
b = \max\left(MAE_{1} - \frac{(MAE_{1} - MAE_{2})l_{2}^{r/d}}{l_{2}^{r/d} - l_{1}^{r/d}} - \frac{2q\Lambda^{'}s}{\sqrt{N}}, 0\right).
\end{equation}
Note that $MAE_{1}$ and $MAE_{2}$ are two practical MAE loss values associated with two ANNs with hidden units $l_{1}$ and $l_{2}$, respectively.

\end{prop}

\begin{proof}
For two ANNs with hidden layers with units $l_{1}$ and $l_{2}$, we set $k$ to $2$ and then estimate their corresponding MAE losses as:
\begin{equation}
\label{eq:v1}
\frac{cq}{l_{1}^{r/d}} + \frac{2 q \Lambda^{'} s}{\sqrt{N}} + b = MAE_{1},
\end{equation}
\begin{equation}
\label{eq:v2}
\frac{cq}{l_{2}^{r/d}} + \frac{2 q \Lambda^{'} s}{\sqrt{N}} + b = MAE_{2}, 
\end{equation}
which can result in hyper-parameters $c$ and $b$. In particular, we substitute $\frac{\mu M^{2} \beta}{2\gamma}$ for $b$ in Eq.~(\ref{eq:core_thm}) and then subtract two sides of Eq.~(\ref{eq:v1}) by Eq.~(\ref{eq:v2}), which can result in Eq.~(\ref{eq:c}). By replacing $c$ in Eq.~(\ref{eq:v1}) with  Eq.~(\ref{eq:c}), we finally obtain Eq.~(\ref{eq:b}).
\end{proof}

Compared with our previous approaches to estimating practical MAE values in \cite{qi2019theory} where the DNN representation power is mainly considered, our new inequality~(\ref{eq:dnn_mse}) arises from the upper bound on the DNN generalization capability such that it can be used to estimate MAE values in more general experimental settings.

\section{Experiments}
\label{sec:exp}
\subsection{Experimental Goals}
Our experiments separately employ the DNN based vector-to-vector regression for both image de-noising and speech enhancement with particular attention to linking empirical results with our proposed theorems. Unlike our analysis of the representation power of the DNN based regression tasks in~\cite{qi2019theory}, this work focuses on the generalization capability of the DNN based vector-to-vector regression based on our derived upper bounds. More specifically, we employ the tasks of image de-noising and speech enhancement, where inconsistent noisy conditions are mixed to the clean training and testing data, to validate our theorems by comparing the estimated MAE upper bound (MAE$\_$B) with the practical ones.

Moreover, the image de-noising experiment corresponds to the ``over-parametrization'' setting in which the number of DNN parameters is much larger than the amount of training data, but we cannot set up the ``over-parametrization'' for speech enhancement tasks due to a significantly large amount of training data. However, we assume the $\gamma$-PL condition and evaluate our derived upper bounds on the speech enhancement tasks. 

Therefore, our experiments of image de-noising and speech enhancement aim at verifying the following points:

\begin{itemize}
\item The estimated MAE upper bound (MAE$\_$B) matches with experimental MAE values.
\item A deeper DNN structure corresponds to a lower approximation error (AE).
\item A significantly small optimization error can be achieved if the ``over-parametrization'' configuration is satisfied. Otherwise, the optimization error could be large enough to dominate MAE losses, even if the $\gamma$-PL condition is assumed. 
\end{itemize}

\subsection{Experiments of Image De-noising}
\label{exp1}

\subsubsection{Data Preparation}
This section presents the image de-noising experiments on the MNIST dataset~\cite{deng2012mnist}. The MNIST dataset consists of $60000$ images for training and $10000$ ones for testing. We added additive Gaussian random noise (AGRN), with mean $0$ and variance $1$, to both training and testing data. The synthesized noisy data were then normalized such that for each image the condition $||\textbf{x}_{noisy}||_{2} \le 1$ is satisfied. 

\subsubsection{Experimental Setup}
The DNN based vector-to-vector regression in the experiments followed a feed-forward neural network architecture, where the inputs were $784$-dimensional feature vectors of the noisy images and the outputs were $784$-dimensional features of either clean or enhanced images. The reference of clean image features associated with the noisy inputs was assigned to the top layer of DNN in the training process, but the top layer corresponds to the features of the enhanced images during the testing stage. Table~\ref{tab:tab0} exhibits the structures of neural networks used in the experiments. In more details, the vector-to-vector regression model was first built based on ANN. The width of the hidden layer of ANN1 was set to $1024$, which satisfies the constraint of the number of neurons in hidden layers based on both the inequality Eq. (\ref{eq:core_thm}) ($d = 784$, $d+2=786 < 1024$) and the ``over-parametrization'' ($784 \times 1024 = 802816 > 60000$) condition. Whereas, ANN2 had a width of $2048$ neurons, which is twice more than ANN1. Next, we studied the DNN based vector-to-vector regression by increasing the number of hidden layers of DNN1. Specifically, DNN1 was equipped with $4$ hidden layers with widths $1024$-$1024$-$1024$-$2048$. Additional two hidden layers of width $1024$ were further appended to DNN2, which brings an architecture with $6$ hidden layers $1024$-$1024$-$1024$-$1024$-$1024$-$2048$. 

\begin{table}[!ht]
\center
\renewcommand{\arraystretch}{1.3}
\caption{Model structures for various vector-to-vector regression}
\label{tab:tab0}
\begin{tabular}{|c||c|}
\hline
 Models & Structures	 (Input -- Hidden\_layers -- Output)	\\
 \hline
\hline
ANN1  & 784-1024-784 	 \\
\hline
ANN2  & 784-2048-784	 \\
\hline
\hline
DNN1  & 784-1024-1024-1024-2048-784		   \\
\hline
DNN2  & 784-1024-1024-1024-1024-1024-2048-784	   \\
\hline
\end{tabular}
\end{table}

Moreover, the SGD optimizer with a learning rate of $0.02$ and a momentum rate of 0.2 was used to update model parameters based on the standard back-propagation (BP) algorithm~\cite{hirose1991back}. The weights of the $k-1$ hidden layers were normalized by dividing the $L_{2}$ norm, which corresponds to the term $\Lambda^{k-1}$ configured to $1$ in the inequality Eq.~(\ref{eq:dnn_mse}). The weights of the top hidden layer were normalized by dividing the $L_{1}$ norm such that $\Lambda'$ is set to $1$. Besides, MAE was employed as the evaluation metric in our experimental validation because the MAE metric is directly connected to the objective loss function of MAE. 

\subsubsection{Experimental Results}
We present our experimental results on the noisy MNIST dataset, where AGRN was added to the clean images. Table~\ref{tab:params0} shows the setup of hyper-parameters $l_{1}$, $l_{2}$, $N$, and $r$ in Eq.~(\ref{eq:dnn_mse}) to estimate $\text{MAE\_B}$. Table~\ref{tab:tab002} exhibits our estimated MAE values are in line with the practical MAE ones. Specifically, DNN2 attains a lower MAE ($0.1278$ vs. $0.1263$) than DNN1. Moreover, our estimated $\text{MAE\_B}$ score for DNN2 is also lower than that for DNN1, namely $0.1438$ vs. $0.1434$, which arises from the decreasing AE score for DNN2 with a deeper architecture. Since we keep $\Lambda$ and $\Lambda'$ equal to $1$, estimation error (EE) and optimization error (OE) for both DNN1 and DNN2 share the same values. Furthermore, although the OE values are comparatively larger than AE and EE, they also stay at a small level because of the ``over-parametrization'' technique adopted in our experiments. 

\begin{table}[!ht]
\center
\renewcommand{\arraystretch}{1.3}
\caption{Hyper-parameters for the estimation of MAE upper bounds.}
\label{tab:params0}
\begin{tabular}{|c|c|c|c|c|c|}
\hline
 $l_{1}$ & $l_{2}$ & $N$ & $r$    & ANN1\_MAE  & ANN2\_MAE\\
\hline
 $1024$  & $2048$ & $6\times 10^{4}$	 &	$1176$ & $0.1318$   & $0.1292$  \\
\hline
\end{tabular}
\end{table}

\begin{table}[!ht]
\center
\renewcommand{\arraystretch}{1.3}
\caption{The evaluation results under the AGRN noise.}
\label{tab:tab002} 
\begin{tabular}{|c|c|c|c|c|c|}
\hline
 Models & MAE & AE & EE & OE &MAE$\_$B \\
\hline
DNN1 & 0.1278  & 0.0172  & 0.0261  & 0.1005	  & 0.1438 \\
\hline
DNN2 & 0.1263   & 0.0168  & 0.0261  & 0.1005  & 0.1434  \\
\hline
\end{tabular}
\end{table}

\subsection{Experiments of Speech Enhancement}
\label{exp2}

\subsubsection{Data Preparation}

Our experiments of speech enhancement were conducted on the Edinburgh noisy speech database~\cite{valentini2016investigating}, where the noisy backgrounds of the training data are inconsistent with the testing ones. More specifically, clean utterances were recorded from $56$ speakers including $28$ males and $28$ females from different accent regions of both Scotland and the United States. Clean material was randomly split into $23075$ training, and $824$ testing waveforms, respectively. The noisy training waveforms at four SNR values, 15dB, 10dB, 5dB, and 0dB, were obtained using the following noises: a domestic noise (inside a kitchen), an office noise (in a meeting room), three public space noises (cafeteria, restaurant, subway station), two transportation noises (car and metro) and a street noise (busy traffic intersection). In sum, there were 40 different noisy conditions for synthesizing many noisy training data (ten noises $\times$ four SNRs). In the noisy test set, the noise types included: a domestic noise (living room), an office noise (office space), one transport (bus) and two street noises (open area cafeteria and a public square), and SNR values included: 17.5dB, 12.5dB, 7.5dB, and 2.5dB. Thus, there were $20$ different noisy conditions for creating the testing dataset.

\subsubsection{Experimental Setup}
The DNN based vector-to-vector regression for speech enhancement also followed the feed-forward ANN architecture, where the input was a normalized log-power spectral (LPS) feature vector~\cite{hou2007saliency, qi2013auditory} of noisy speech, and the output was LPS feature vector of either clean or enhanced speech. The references of clean speech feature vectors associated with the noisy inputs were assigned to the top layer of DNN in the training process, but the top layer of DNN corresponds to the feature vectors of the enhanced speech during the testing phase. The smooth ReLU function in Eq.~(\ref{eq:relu}) was employed in the hidden nodes of the neural architectures assessed in this work, whereas a linear function was used at the output layer. To improve the subjective perception in the speech enhancement tasks, the global variance equalization~\cite{toda2005spectral} was applied to alleviate the problem of over-smoothing by correcting a global variance between estimated features and clean reference targets~\cite{qi2013subspace}. In the training stage, the BP algorithm was adopted to update the model parameters, and the MAE loss was used to measure the difference between a normalized LPS vector, and the reference one. Noise-aware training (NAT)~\cite{seltzer2013investigation} was also employed to enable non-stationary noise awareness, and feature vectors of $3$-frame size were obtained by concatenating frames within a sliding window~\cite{qi2013bottleneck}. Moreover, the SGD optimizer with a learning rate of $1\times 10^{-3}$ and a momentum rate of $0.4$ was used for the update of parameters. The weights of the first $k-1$ hidden layers are normalized by dividing the $L_{2}$ norm of each row of weights, which correspond to the term $\Lambda^{k-1}$ equal to $1$ in Eq.~(\ref{eq:core_thm}). Moreover, we set $s$ in Eq.~(\ref{eq:core_thm}) as the maximum value of $L_{2}$ norm of the input, and assume $\Lambda^{'}$ in Eq.~(\ref{eq:core_thm}) as the maximum value of $(||\textbf{W}_{k}(1, :)||_{1}, ..., ||\textbf{W}_{k}(q, :)||_{1})$, which are different from the setup of image de-noising.

Table~\ref{tab:tab20} exhibits the architectures of neural networks used in our experiments of speech enhancement. Two ANN models (ANN1 and ANN2) were utilized to estimate the hyper-parameters in Eq.~(\ref{eq:dnn_mse}), which were then used to estimate the MAE values of DNN models based on Eq.~(\ref{eq:dnn_mse}).

\begin{table}[!ht]
\center
\renewcommand{\arraystretch}{1.3}
\caption{Model structures for various vector-to-vector regression}
\label{tab:tab20}
\begin{tabular}{|c||c|}
\hline
 Models & Structures	 (Input -- Hidden\_layers -- Output)	\\
 \hline
\hline
ANN1  & 771-800-257  	 \\
\hline
ANN2  & 771-1600-257	 \\
\hline
\hline
DNN1  & 771-800-800-800-1600-257			   \\
\hline
DNN2  & 771-800-800-800-800-800-1600-257	   \\
\hline
\end{tabular}
\end{table}

Two evaluation metrics, namely MAE and Perceptual Evaluation of Speech Quality (PESQ)~\cite{rix2001perceptual}, were employed in our experimental validation. Different from the MAE metric, PESQ is an indirect evaluation which is highly correlated with speech quality. The PESQ score, which ranges from $-0.5$ to $4.5$, is calculated by comparing the enhanced speech with the clean one. A higher PESQ score corresponds to a higher quality of speech perception. All of the evaluation results on testing datasets are listed in Tables~\ref{tab:tab7}.

\subsubsection{Experimental Results}

We now present our experimental results on the Edinburgh speech database. Table~\ref{tab:params2} shows that the parameters used in the experiments to estimate the upper bound based on the inequality~(\ref{eq:dnn_mse}). The experimental results as shown in Table~\ref{tab:tab7} are in line with those observed in the consistent noisy conditions. Specifically, DNN2 attains a lower MAE (0.6859 vs. 0.7060) and higher PESQ values (2.85 vs. 2.82) than DNN1. Moreover, the MAE$\_$B score for DNN2 is also lower than that for DNN1, namely 0.7124 vs. 0.7236. Furthermore, DNN2 owns a better representation power in terms of AE scores (0.0081 vs. 0.0161) and a better power generalization capability because of a lower (EE + OE) score. More significantly, the OE term actually is the key contributor to the MAE$\_$B score, which suggests that the MAE loss is primarily from OE, as expected. In fact, optimization plays an important role when it comes to training large neural architectures \cite{xu2015regression, xu2013experimental}, which in turn shows that the proposed upper bounds are in line with current research efforts \cite{chizat2018global, allen2018convergence, vaswani2018fast, allen2018learning} on the optimization strategies.

\begin{table}[!ht]
\center
\renewcommand{\arraystretch}{1.3}
\caption{Hyper-parameters for the estimation of MAE upper bounds.}
\label{tab:params2}
\begin{tabular}{|c|c|c|c|}
\hline
$l_{1}$ 	& $l_{2}$ &  $N$ & $r$    	 \\
\hline
 800	& 1600 & $1.04\times 10^{10}$ &  771 	\\
\hline
ANN1\_MAE &  ANN2\_MAE & $\Lambda' $(ANN1)    & $\Lambda'$ (ANN2)	\\
\hline
0.7409  &  0.7328	&8.9543	&  10.1542	\\
\hline
\end{tabular}
\end{table}

\begin{table}[!ht]
\center
\renewcommand{\arraystretch}{1.3}
\caption{The MAE Results on the Edinburgh speech database}
\label{tab:tab7}
\begin{tabular}{|c|c|c|c|c|c|c|}
\hline
 Models & MAE & PESQ &   AE & EE &OE &MAE$\_$B   \\
\hline
DNN1   & 0.7060  	 &  2.82 	& 0.0161	& 0.0579	&  0.6496	 & 	0.7236	 \\
\hline
DNN2   & 0.6859	 &  2.85 	&  0.0081	& 0.0728	 & 0.6315  & 	0.7124	\\
\hline
\end{tabular}
\end{table}

\subsection{Discussions}

The experimental results of the image de-noising and speech enhancement suggest that our proposed upper bounds on the generalized loss of MAE can tightly estimate the practical MAE values. Unlike our previous work on the analysis of the representation power, which is strictly constrained to consistent noisy environments, our MAE bounds aim at the generalization power of DNN based vector-to-vector regression and can be generalized to more general noisy conditions.

Experimental results are based on our aggregated bound in Eq.~(\ref{eq:core_thm}), and the related practical methods in Eq. (\ref{eq:dnn_mse}). The decreasing AE scores of DNN2 correspond to Eq.~(\ref{eq:con1}), where a deeper depth $k$ can lead to smaller AE values. In the meanwhile, Eqs.~(\ref{eq:c}) and~(\ref{eq:b}) suggest that a smaller EE is associated with a larger OE, which also corresponds to our estimated results. Furthermore, deeper DNN structures can result in a larger $\Lambda'$, which slightly escalates the AE scores and also decreases OE values. With the setup of ``over-parametrization'' for neural networks in image de-noising experiments, OE can be lowered to a small scale compared to AE and EE. However, OE becomes much lager than AE and EE without the ``over-parametrization'' configuration in the speech enhancement tasks. 

\section{Conclusion}
\label{sec:con}
This study focuses on the theoretical analysis of an upper bound on a generalized loss of MAE for DNN based vector-to-vector regression and corroborates the theorems with image de-noising and speech enhancement experiments. Our theorems start from decomposing a generalized MAE loss, which can be upper bounded by the sum of approximation, estimation, and optimization errors for the DNN based vector-to-vector functions. Our previously proposed bound on the representation power of DNNs can be modified to upper bound the approximation error, and a new upper bound based on the empirical Rademacher complexity is derived for the estimation error. Furthermore, the smooth modification of the ReLU function and the assumption of $\gamma$-PL conditions under the ``over-parametrization'' configuration for neural networks can ensure an upper bound on the optimization error. Thus, an aggregated upper bound for MAE can be derived by combining the upper bounds for the three errors. 

Our experimental results of image de-noising and speech enhancement show that a deeper DNN with a broader width at the top hidden layer can contribute to better generalization capability in various noisy environments. The estimated MAE based on our related theorems can offer tight upper bounds for the experimental values in practice, which can verify our theorem of upper bounding MAE for DNN based vector-to-vector regression. Besides, our theories show that the MAE value mainly arises from the optimization error for well-regularized DNNs, and an ``over-parametrization'' for neural networks can ensure small optimization errors. 

\appendices
\section{}
\label{sec:appA}

\begin{lemma}
\label{lemma3}
Let $\mathcal{\hat{L}}\in \mathbb{L}$ denote the loss function for $N$ samples $S = \{\textbf{x}_{1}, \textbf{x}_{2}, ..., \textbf{x}_{N}\}$ drawn i.i.d. according to a distribution $\mathcal{D}$. For a generalized MAE loss function $\mathcal{L} \in \mathbb{L}$, we have that 
\begin{equation}
 \mathbb{E}\left[\sup\limits_{f_{v}\in \mathbb{F}} \left|\mathcal{L}(f_{v}) - \mathcal{\hat{L}}(f_{v}) \right|\right]  \le 2 \mathcal{\hat{R}}_{S}(\mathbb{F}).
\end{equation}
\end{lemma}

\begin{proof}
We utilize the method of symmetrization~\cite{vershynin2018high} to bound $ \mathbb{E}\left[\sup_{f\in \mathbb{F}} \left|\mathcal{L}(f_{v}) - \mathcal{\hat{L}}(f_{v}) \right|\right]$. The symmetrization introduces a ghost dataset $S' = \{\textbf{x}'_{1}, \textbf{x}'_{2}, ..., \textbf{x}'_{N}\}$ drawn i.i.d. from $D$. Let $\mathcal{\hat{L}}'\in \mathbb{L}$ be the empirical risk with respect to the ghost dataset, and we assume $\mathcal{L}(f_{v}) = \mathbb{E}_{S'}[\mathcal{\hat{L}}'(f_{v})]$. Assuming $\mathcal{L}(f_{v}) \ge \mathcal{\hat{L}}(f_{v}), \forall f_{v} \in \mathbb{F}$, we derive that
\begin{equation*}
\begin{split}
&\hspace{5mm} \mathbb{E}_{S}\left[ \sup\limits_{f_{v} \in \mathbb{F}} \left| \mathcal{L}(f_{v}) - \mathcal{\hat{L}}(f_{v}) \right| \right]		\\
&=  \mathbb{E}_{S}\left[ \sup\limits_{f_{v}\in\mathbb{F}} \left( \mathcal{L}(f_{v}) - \mathcal{\hat{L}}(f_{v}) \right) \right] \\
&=  \mathbb{E}_{S}\left[ \sup\limits_{f_{v} \in \mathbb{F}} \left(  \mathbb{E}_{S'}[\mathcal{\hat{L}}'(f_{v})] - \mathcal{\hat{L}}(f_{v}) \right) \right]	\\
&\le \mathbb{E}_{S}\left[ \mathbb{E}_{S'}\left[ \sup\limits_{f_{v} \in \mathbb{F}} \left(  \mathcal{\hat{L}}'(f_{v}) - \mathcal{\hat{L}}(f_{v}) \right)\right] \right]		\\
&\le \mathbb{E}_{S, S'}\left[  \sup\limits_{f_{v}\in \mathbb{F}} \frac{1}{N} \sum\limits_{i=1}^{N} \sigma_{i} (\mathcal{\hat{L}}'(f_{v}(\textbf{x}'_{i})) - \mathcal{\hat{L}}(f_{v}(\textbf{x}_{i}))) \right]	\\
&\le 2 \mathcal{\hat{R}}_{S}(\mathbb{F}), 
\end{split}
\end{equation*}
where $\sigma_{1}, \sigma_{2}, ..., \sigma_{N}$ are Rademacher random variables. Similarly, the assumption of $\mathcal{L}(f_{v}) \le \mathcal{\hat{L}}(f_{v}), \forall f_{v} \in \mathbb{F}$ also brings the same result. Thus, we finish the proof of Lemma~\ref{lemma3}. 
\end{proof}

\begin{lemma}[An extension of empirical Rademacher identities]
\label{lemma2}
Given any sample set $S = \{\textbf{x}_{1}, \textbf{x}_{2}, ..., \textbf{x}_{N}\}$, and hypothesis sets $\mathbb{F}_{1}$, $\mathbb{F}_{2}$, ..., $\mathbb{F}_{q}$ of functions $f_{v}^{(1)}\in \mathbb{F}_{1}$, $f_{v}^{(2)}\in \mathbb{F}_{2}$, ..., $f_{v}^{(q)}\in \mathbb{F}_{q}$ mapping from $\mathbb{R}^{d}$ to $\mathbb{R}^{q}$, we have that 
\begin{equation*}
\begin{split}
 \mathcal{\hat{R}}_{S}(\sum\limits_{i=1}^{q}\mathbb{F}_{i}) &= \frac{1}{N} \mathbb{E}_{\sigma}\left[ \sup\limits_{f_{v}^{(1)}\in\mathbb{F}_{1}, ..., f_{v}^{(q)}\in \mathbb{F}_{q}} \sum\limits_{i=1}^{N}\sigma_{i} \left(\sum\limits_{j=1}^{q}f_{v}^{(j)}(\textbf{x}_{i})\right) \right]		\\
 &= \frac{1}{N}\sum\limits_{j=1}^{q} \mathbb{E}_{\sigma}\left[ \sup\limits_{f_{v}^{(1)}\in\mathbb{F}_{1}, ..., f_{v}^{(q)}\in \mathbb{F}_{q}} \sum\limits_{i=1}^{N}\sigma_{i} f_{v}^{(j)}(\textbf{x}_{i}) \right] \\
 &= \sum\limits_{i=1}^{q} \mathcal{\hat{R}}_{S}(\mathbb{F}_{i}).
 \end{split}
\end{equation*}
\end{lemma}

\section{}
\begin{figure}[htbp]
\centerline{\epsfig{figure=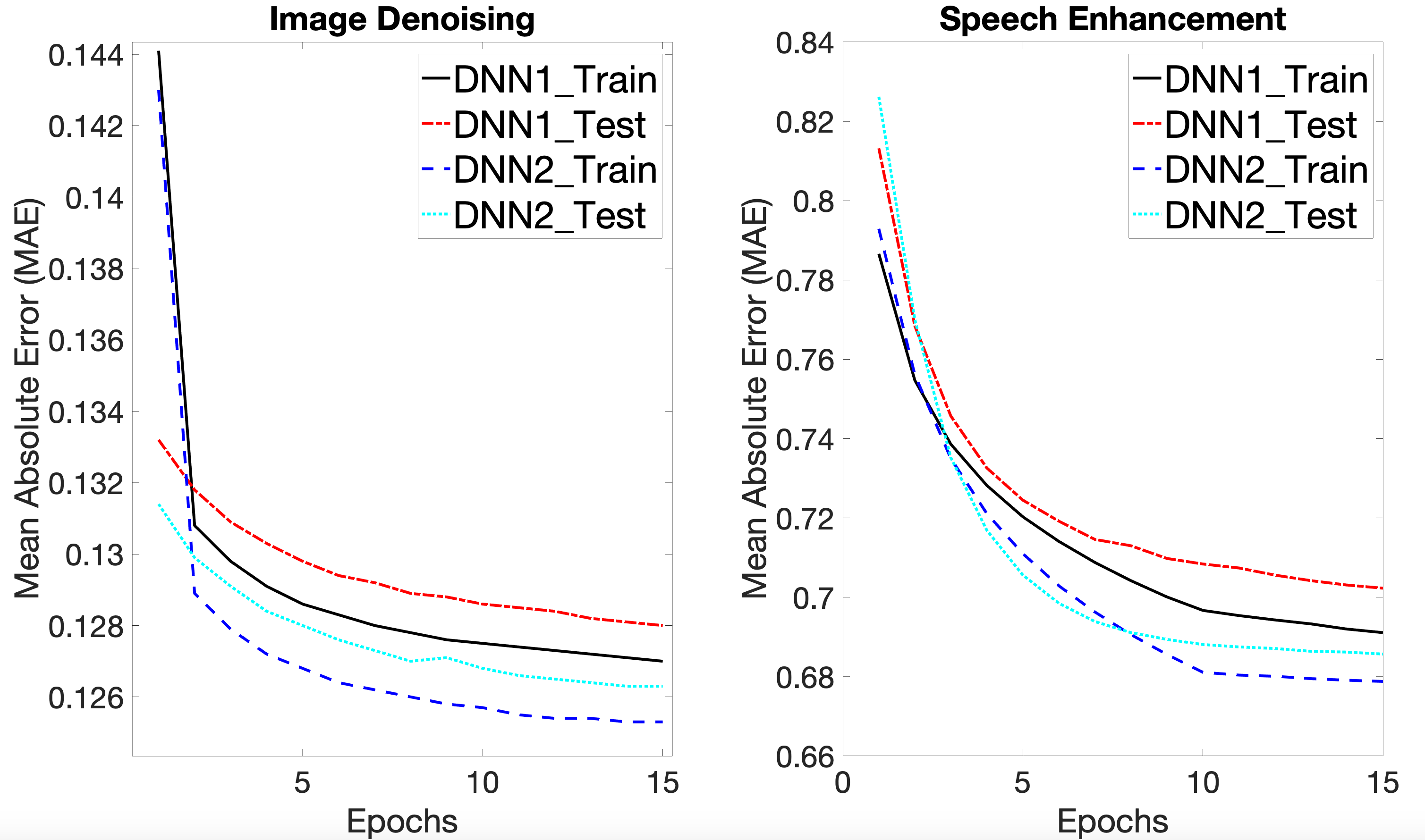, width=85mm}}
\caption{{\it Training and testing MAE curves over epochs in the two experiments.}}
\label{fig:rate1}
\end{figure}
Figure~\ref{fig:rate1} illustrates both training and testing curves of MAE over epochs in our experiments of image de-noising and speech enhancement. The simulations suggest that DNN2 with deeper architectures consistently achieves lower training and testing MAE values than DNN1 over epochs. When the update of model parameters gets converged, DNN2 finally attains the best performance on the two experiments. 

\ifCLASSOPTIONcaptionsoff
  \newpage
\fi

\bibliographystyle{IEEEtran}
\bibliography{speech}

\begin{IEEEbiography}[{\includegraphics[width=1in,height=1.25in,clip,keepaspectratio]{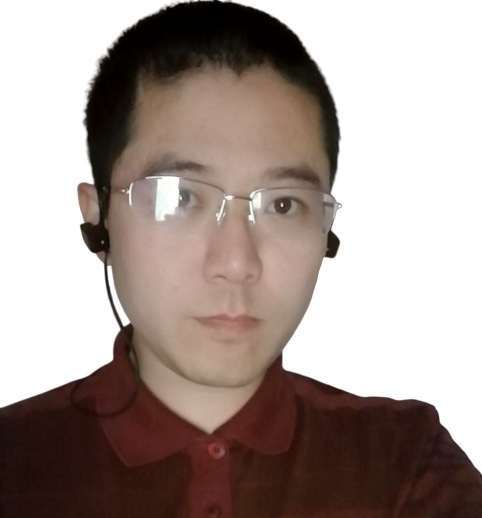}}]{Jun Qi} is a Ph.D. candidate at School of Electrical and Computer Engineering, Georgia Institute of Technology, Atlanta. Previously, he completed a graduate study in Electrical Engineering with an MSEE at the University of Washington, Seattle in 2017 and another Master of Engineering at Tsinghua University, Beijing in 2013. Also, he was a research intern in Deep Learning Technology Center (DLTC) at Microsoft Research, Redmond, WA in 2017, a research intern at Tencent American AI Lab, Bellevue, WA in 2019, and a graduate research intern at Mitsubishi Electric Research Laboratory (MERL), Cambridge, MA in 2020. His research focuses on (1) Leveraging Non-Convex Optimization and Statistical Learning Theory for Analyzing Deep Learning based Signal Processing; (2) Tensor Decomposition and Submodular Optimization applied to Speech and Natural Language Processing; (3) Theoretical Reinforcement Learning and Applications in Quantitative Finance. 
\end{IEEEbiography}

\begin{IEEEbiography}[{\includegraphics[width=1in,height=1.25in,clip,keepaspectratio]{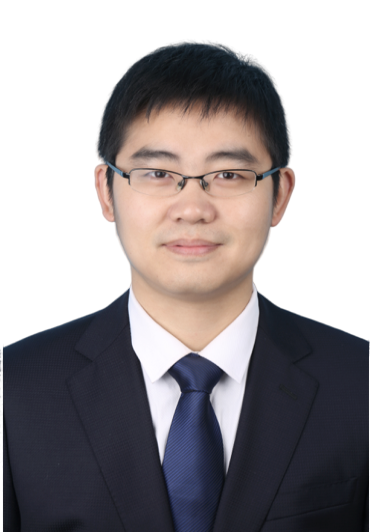}}]{Jun Du} received the B.Eng. and Ph.D. degrees from the Department of Electronic Engineering and Information Science, University of Science and Technology of China (USTC) in 2004 and 2009, respectively. From 2004 to 2009, he was with iFlytek Speech Lab of USTC. During the above period, he was am Intern at Microsoft Research Asia (MSRA), Beijing. In 2007, he  worked as a Research Assistant for 6 months  at the University of Hong Kong. From July 2009 to June 2010, he was with  iFlytek Research. From July 2010 to January 2013, he joined MSRA as an Associate Researcher. Since February 2013, he has been with the National Engineering Laboratory for Speech and Language Information Processing (NEL-SLIP) of USTC.
\end{IEEEbiography}

\begin{IEEEbiography}[{\includegraphics[width=1in,height=1.25in,clip,keepaspectratio]{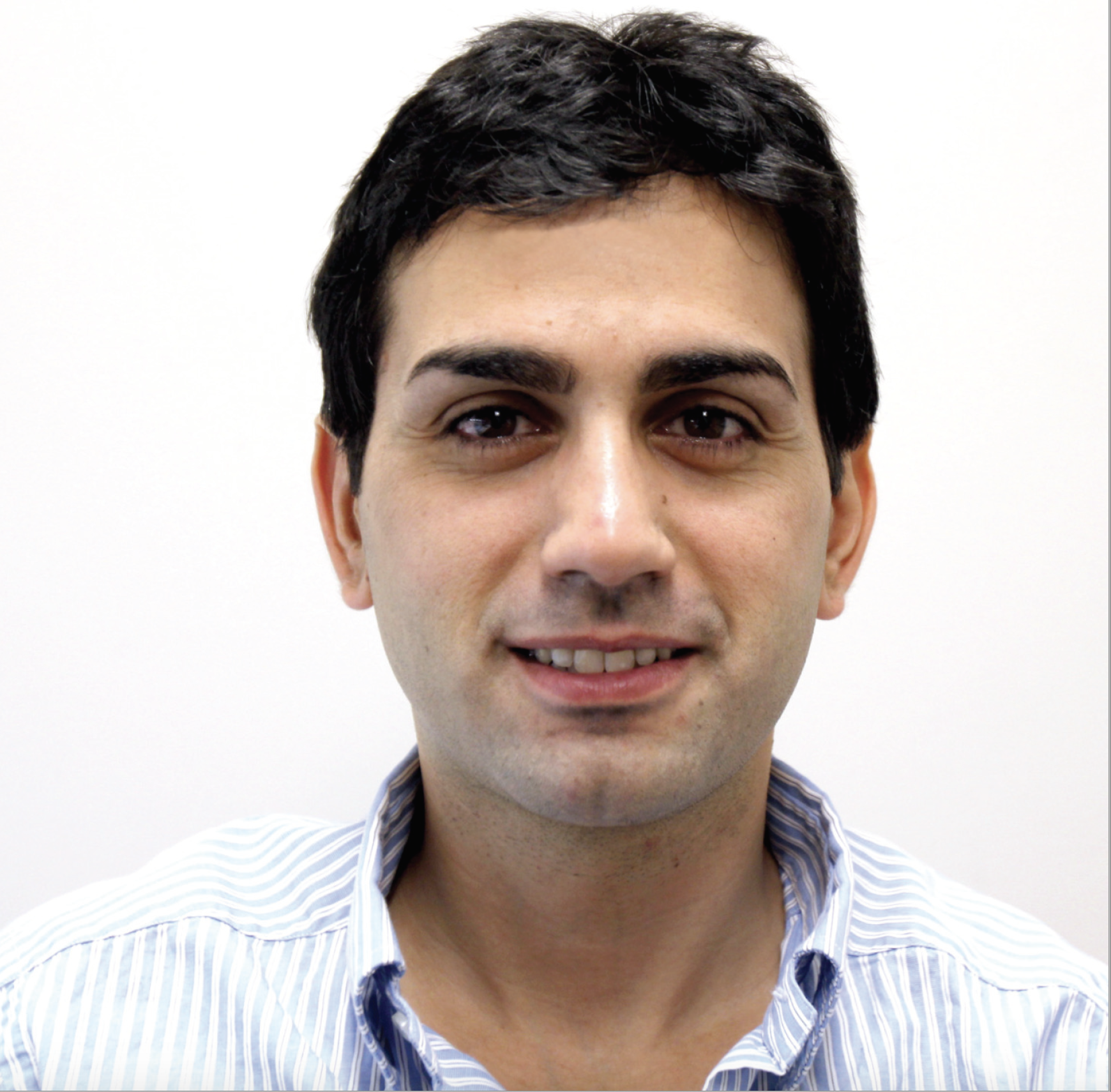}}]{Sabato Marco Siniscalchi} is a Professor at the University of Enna  and  affiliated with the Georgia Institute of Technology. He received his Laurea and Doctorate degrees in Computer  Engineering from the University of Palermo, Italy, in 2001 and 2006, respectively.  In 2006, he was a Post Doctoral Fellow at the Georgia Institute of Technology,USA. From 2007 to 2009, he joined the Norwegian University of Science and Technology (NTNU) as a Research Scientist. From 2010 to 2015, he was an Assistant Professor, first, and an Associate Professor, after, at the University of Enna. From 2017 to 2018, he joined as s Senior Speech Researcher the Siri Speech Group, Apple Inc., Cupertino CA, USA.
\end{IEEEbiography}

\begin{IEEEbiography}[{\includegraphics[width=1in,height=1.25in,clip,keepaspectratio]{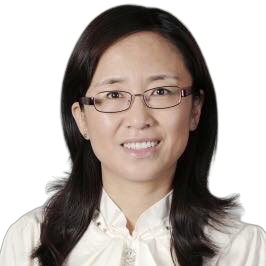}}]{Xiaoli Ma} (S'10--M'16) is Professor in School of Electrical and Computer Engineering at Georgia Institute of Technology, Atlanta, GA, USA. She is an IEEE Fellow for her contributes to “block transmissions over wireless fading channels.” Her research interests are in the areas of signal processing for communications and networks, signal estimation algorithms, coding theory, wireless communication theory, and sensor and ad hoc networks. Ma is a senior area editor for IEEE Signal Processing Letters and Elsevier Digital Signal Processing and has been an associate editor for the IEEE Transactions on Wireless Communications and Signal Processing Letters. Her recent research interests rely on intelligent wireless communication, decentralized networks, and sensor networks.
\end{IEEEbiography}

\begin{IEEEbiography}[{\includegraphics[width=1in,height=1.25in,clip,keepaspectratio]{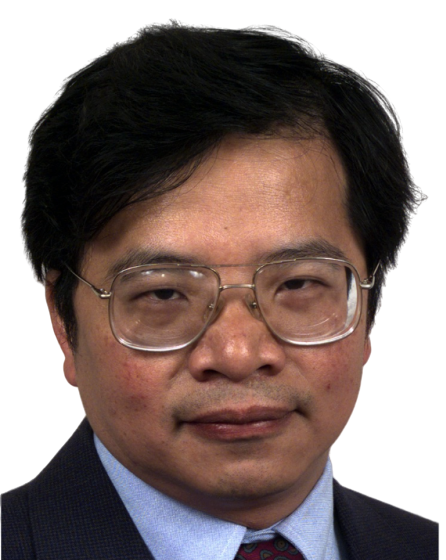}}]{Chin-Hui Lee}
is a professor at School of Electrical and Computer Engineering, Georgia Institute of Technology. Before joining academia in 2001, he had 20 years of industrial experience ending in Bell Laboratories, Murray Hill, New Jersey, as a Distinguished Member of Technical Staff and Director of the Dialogue Systems Research Department. Dr. Lee is a Fellow of the IEEE and a Fellow of ISCA. He has published over 400 papers and 30 patents, and was highly cited for his original contributions with an h-index of 66. He received numerous awards, including the Bell Labs President Gold Award in 1998. He won the SPS 2006 Technical Achievement Award for Exceptional Contributions to the Field of Automatic Speech Recognition.
\end{IEEEbiography}

\end{document}